\documentclass[conference]{IEEEtran}

\usepackage{fancyhdr}

\usepackage[utf8]{inputenc} 
\usepackage[T1]{fontenc}    
\usepackage{hyperref}       
\usepackage{url}            
\usepackage{booktabs}       
\usepackage{amsfonts}       
\usepackage{nicefrac}       
\usepackage{microtype}      

\usepackage{pgfplots}
\usetikzlibrary{pgfplots.groupplots}
\usepackage{graphicx}

\usepackage{amsmath}
\usepackage{amsfonts}
\usepackage{amsthm}
\usepackage{amssymb}
\usepackage{color}
\usepackage{xspace}
\usepackage{multirow}
\usepackage{array}
\usepackage{caption}
\usepackage{subcaption}
\usepackage{ragged2e}

\usepackage[all=normal, floats=tight, paragraphs=tight]{savetrees}

\def\showOutline{1}

\newcommand{\outline}[1]{
  \if\showOutline 1
    \textbf{[outline]} \textit{#1}
  \fi
}

\newcommand{\TODO}[1]{
  \if\showOutline 1
    \textit{\color{blue} (TODO: #1)}
  \fi
}

\theoremstyle{plain}
\newtheorem{lemma}{Lemma}
\newtheorem{axiom}{Axiom}
\newtheorem{theorem}{Theorem}
\newtheorem{defn}{Definition}

\newcommand{\infl}{\chi}
\newcommand{\pderiv}[2]{\frac{\partial #1}{\partial #2}}
\newcommand{\pderivat}[3]{\left. \pderiv{#1}{#2}\right\rvert_{#3}}
\renewcommand{\vec}[1]{\mathbf{#1}}
\newcommand{\x}{\vec{x}}
\newcommand{\X}{\mathcal{X}}
\newcommand{\R}{\mathbb{R}}
\newcommand{\z}{\vec{z}}
\newcommand{\Z}{\mathcal{Z}}

\newcommand{\influence}{distributional influence\xspace}
\newcommand{\Influence}{Distributional influence\xspace}
\newcommand{\axla}{linearity agreement\xspace}

\newcommand{\cut}[1]{ }

\title{Influence-Directed Explanations \\ for Deep Convolutional Networks}

\author{
  \IEEEauthorblockN{
    Klas Leino\\
    Shayak Sen\\
    Anupam Datta\\
    and Matt Fredrikson
  }
  \IEEEauthorblockA{Carnegie Mellon University}
  \and
  \IEEEauthorblockN{Linyi Li}
  \IEEEauthorblockA{Tsinghua University}
}

\begin{document}
\pgfplotsset{
    every non boxed x axis/.style={}
}

\maketitle

\begin{abstract}
  We study the problem of explaining a rich class of behavioral properties of deep neural networks.
Distinctively, our \emph{influence-directed explanations} approach this problem by peering inside the network to identify neurons with high \emph{influence} on a quantity and distribution of interest, using an axiomatically-justified influence measure, and then providing an \emph{interpretation} for the concepts these neurons represent.
We evaluate our approach by demonstrating a number of its unique capabilities on convolutional neural networks trained on ImageNet.
Our evaluation demonstrates that influence-directed explanations
\emph{(1)} identify influential concepts that generalize across instances,
\emph{(2)} can be used to extract the ``essence'' of what the network learned about a class, 
and \emph{(3)} isolate individual features the network uses to make decisions and distinguish related classes.

\end{abstract}

\section{Introduction}
\label{sec:intro}
We study the problem of explaining a class of behavioral properties of deep neural networks, with a focus on convolutional neural networks. This problem has received significant attention in recent years with the rise of deep networks and associated concerns about their opacity~\cite{ai-black-box}.

A growing body of work on explaining deep convolutional network behavior is based on mapping models' prediction outputs back to relevant regions in an input image. This is accomplished in various ways, such as by visualizing gradients~\cite{saliency,integratedGrads,bach-plos15}, or by backpropagation~\cite{deconv,backprop,bach-plos15}. 
An appealing feature of these approaches is that they capture \textit{input influence}. However, because these approaches relate instance-specific features to instance-specific predictions, the explanations that they produce do not generalize beyond a single test point (see Section~\ref{sec:exp:effectiveness}, Figure~\ref{fig:meanvindividual}).

An orthogonal approach is to visualize the features learned by networks by identifying input instances that maximally activate an internal neuron, by either optimizing the activation in the input space~\cite{saliency,MahendranV14,NguyenDYBC16}, or searching for instances in a dataset~\cite{girshick14rich}. Importantly, this type of explanation gives insight into the higher-level concepts learnt by the network, and naturally generalizes across instances and classes. However, this approach does not relate these higher-level concepts to the predictions that they cause. Indeed, examining activations alone is not sufficient to do so (see Section~\ref{exp:model_comp}).

This paper introduces \emph{influence-directed explanations} for deep networks to combine the positive attributes of these two lines of work. Our approach peers inside the network to identify neurons with high \emph{influence} on the model's behavior, and then uses existing techniques (e.g., visualization) to provide an \emph{interpretation} for the concepts they represent. We introduce a novel \emph{distributional influence} measure that allows us to identify which neurons are most influential in determining the model's behavior on a given distribution of instances. From this we are able to identify the learned concepts that cause the network to behave characteristically, for example, on the distribution of instances that share a particular label.

\begin{figure}[t]
\centering
\begin{subfigure}{0.25\textwidth}
\centering
\begin{tabular}{@{\hskip 0pt}c@{\hskip 1pt}c@{\hskip 1pt}c@{\hskip 0pt}}
\includegraphics[width=0.3\textwidth]{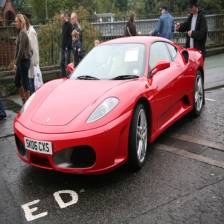}
&
\includegraphics[width=0.3\textwidth]{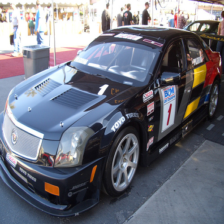}
&
\includegraphics[width=0.3\textwidth]{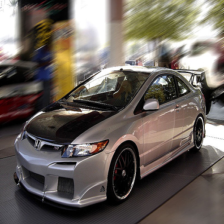}
\\[-0.6ex]
\includegraphics[width=0.3\textwidth]{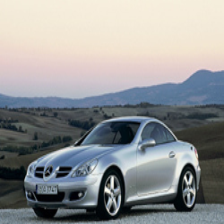}
&
\includegraphics[width=0.3\textwidth]{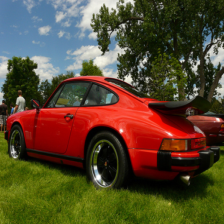}
&
\includegraphics[width=0.3\textwidth]{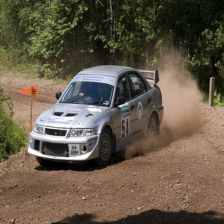}
\\[-0.6ex]
\includegraphics[width=0.3\textwidth]{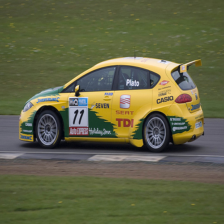}
&
\includegraphics[width=0.3\textwidth]{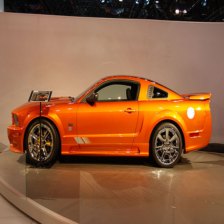}
&
\includegraphics[width=0.3\textwidth]{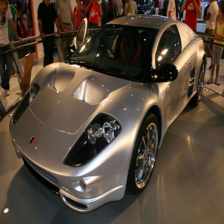}
\\
\end{tabular}
\caption{ }
\label{fig:concepts_orig}
\end{subfigure}%
\begin{subfigure}{0.25\textwidth}
\centering
\begin{tabular}{@{\hskip 0pt}c@{\hskip 1pt}c@{\hskip 1pt}c@{\hskip 0pt}}
\includegraphics[width=0.3\textwidth]{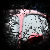}
&
\includegraphics[width=0.3\textwidth]{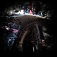}
&
\includegraphics[width=0.3\textwidth]{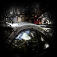}
\\[-0.6ex]
\includegraphics[width=0.3\textwidth]{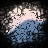}
&
\includegraphics[width=0.3\textwidth]{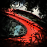}
&
\includegraphics[width=0.3\textwidth]{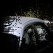}
\\[-0.6ex]
\includegraphics[width=0.3\textwidth]{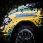}
&
\includegraphics[width=0.3\textwidth]{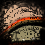}
&
\includegraphics[width=0.3\textwidth]{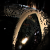}
\\
\end{tabular}
\caption{ }
\label{fig:concepts_expl}
\end{subfigure}
\caption{\textbf{(a)} Images of cars labeled `sports car' by the VGG16 network, and \textbf{(b)} receptive fields of the most influential feature map on a comparative quantity that characterizes the model's tendency to predict `sports car' over `convertible.'}
\label{fig:concepts}
\end{figure}

Figure~\ref{fig:concepts} demonstrates the capability of influence-directed explanations to extract meaningful insight about the network's inner workings. We measure the influence of feature maps at the \verb'conv4_1' layer on the network's tendency to predict `sports car' over `convertible.' The images in Figure~\ref{fig:concepts_expl} are computed by rendering the receptive field of the \emph{most} influential map in the original feature space for the corresponding image in Figure~\ref{fig:concepts_orig}. The results coincide with an intuitive understanding of the distinction between these classes: in most instances, the depicted interpretation highlights the portion of the image depicting the car's top.

Our empirical evaluation demonstrates that influence-directed explanations \emph{(1)} extract influential concepts that generalize across instances, whereas those computed using input influence fail to do so (Section~\ref{sec:exp:effectiveness}), 
\emph{(2)} reveal the ``essence'' of how the network views a class and distinguishes it from others (Section~\ref{exp:model_comp}),
and \emph{(3)} isolate high-level features that the network uses to make predictions (Section~\ref{sec:individual-slice}, \ref{sec:comparative}).
In each case, our influence-directed explanations leverage the ability to measure internal influence to produce useful explanations that would not have been possible otherwise.

\section{Influence}
\label{sec:influence}

In this section, we propose \emph{\influence}, an axiomatically-justified family of influence measures.
\Influence is parameterized by  a \emph{slice} $s$ of the network (e.g. a particular layer), a \emph{quantity of interest} $f$ and a \emph{distribution of interest} $P$.
Given these elements, we measure influence as the partial derivative of $f$ at the slice $s$ averaged over $P$.  
We describe the measure and its parameters in more detail below. 
In Section~\ref{sec:axioms}, we justify this family of measures by proving that these are the only measures that satisfy 
some natural properties.

The slice parameter exposes the internals of a network, allowing us to measure influence with respect to intermediate neurons.
This is a significant departure from prior work, and is key to our goal of identifying high-level concepts that are learned by a network.
As we show in Section~\ref{sec:experiments}, influence measurements on internal units lead to explanations that generalize across instances.
This is usually not possible by measuring input features (i.e., pixels), as learned concepts can manifest themselves in many different ways in the input space, with a high degree of variance among the influence of particular input features across instances.

The distribution and quantity of interest together capture aspects of network behavior that we are interested in explaining.  
Examples of distributions of interest are: \emph{(1)} a single instance (i.e., the influence measure just reduces to the gradient at the point); \emph{(2)} the distribution of `cat' images, or \emph{(3)} the distribution of all images in a dataset. 
While the first distribution of interest focuses on why a single instance was classified a particular way, the second explains the ``essence'' of a class, and the third identifies generally-influential neurons over the entire population. 
Another example is the uniform distribution on the line segment of scaled instances between an instance and a baseline, which yields a measure similar to one called Integrated Gradients~\cite{integratedGrads}.

Whereas the distribution of interest identifies the subjects of an explanation, the quantity of interest identifies the question that is being addressed.
For example, the quantity of interest may correspond to the network's outcome for the `cat' class, or its \emph{comparative} outcome towards `cat' versus `dog' (i.e., the difference in the network scores for cat and dog classes).  
The first quantity addresses the question of why a particular input is classified as `cat', whereas the second addresses how the network distinguishes `cat' instances from `dog' instances.

We represent quantities of interest of networks as continuous and
differentiable functions $f : \X \to \R$ where $\X \subseteq \R^n$ and
$n$ is the number of inputs to $f$. A \influence measure, denoted by
$\infl_i(f, P)$, measures the influence of input $i$ for a quantity of
interest $f$, and a distribution of interest $P$ where $P$ is a distribution
over $\X$.

Next, we define a slice of a network. A particular layer in the network can be
viewed as a slice. More generally, a slice is any partitioning of the network
into two parts that exposes its internals.  Formally, a slice $s$ of a network
$f$ is a tuple of functions $\langle g, h\rangle$ such that $h : \X \to \Z$, $g : \Z \to \R$ and $f = g \circ h$.  
The internal representation for an
instance $\vec x$ is given by $\vec{z} = h(\vec{x})$. In our setting, elements
of $\vec z$ can be viewed as the activations of neurons at a particular layer.

\begin{defn}
\label{defn:infl-slice}
The influence of an element $j$ in the internal representation defined by $s =
\langle g, h\rangle$ is

\begin{equation}
\label{eq:infl-slice}
\infl_j^s(f,P) = \int_{\X}\pderivat{g}{z_j}{h(\x)}P(\x)d\x
\end{equation}
\end{defn}

\section{Identifying Influential Concepts}
\label{sec:experiments}




\begin{figure}[t]
\centering
\begin{minipage}{.5\textwidth}
\centering
\input{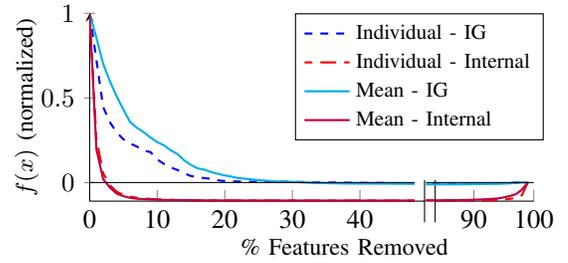}
\end{minipage}%
\caption{
Plot of the decrease in the function value, $f(x)$, as features are removed from the input (\textbf{IG}) or first fully-connected layer (\textbf{Int}) in order of influence, using the VGG16~\cite{Simonyan14c} network. The vertical axis was normalized so that the average value of $f(x)$ is 1, and the average value of $f(0)$ is 0.
The dashed curves depict the average quantity when influence is measured for each instance individually, and the solid curves when the mean influence over the respective class is used for each instance.
Plots are averaged across 5 randomly-selected ImageNet classes.
}
\label{fig:meanvindividual}
\end{figure}

The influence measure defined in Section~\ref{sec:influence} is parameterized by a distribution of interest $P$ (Equation~\ref{eq:infl-slice}) over which the measure is taken. 
By selecting $P$ to be a point mass, the resulting measurements characterize the importance of features for the model’s behavior on a single instance. 
Defining the distribution of interest with support over a larger set of instances yields explanations that capture the factors common to network behaviors across the corresponding population of instances. 
In this section, we demonstrate that when taken at a high internal layer, distributional influence identifies concepts that generalize well across instances. 
Furthermore, we show this measure often lets us identify a relatively small set of concepts that characterize the ``essence'' of the class, and are sufficient for distinguishing instances of that class from others.

\subsection{Effectiveness of Internal Influence}
\label{sec:exp:effectiveness}

One of our central claims is that the ability to measure internal influence across an appropriately chosen distribution lets us identify learned concepts that are relevant to classification predictions.
Figure~\ref{fig:meanvindividual} quantifies the degree to which internal units identified using internal influence measurements correspond to relevant general concepts, compared against the influence measurements obtained using integrated gradients (IG)~\cite{integratedGrads}.
The curves report the network's output at the coordinate of the predicted class, normalized to begin at 1, as input features (IG) or internal units at the lowest fully-connected layer are ``turned off'' in decreasing order of influence.
We adapted this approach from Samek et al.~\cite{Samek2017} for internal units by setting their activation to 0. 
The vertical axis depicts the dropoff of the network’s output against the percentage of features that have been removed.

We evaluated this measure on instances of five randomly-selected ImageNet classes on VGG16~\cite{Simonyan14c}, and display the averaged results. 
We selected integrated gradients as our point of comparison because we found that it outperformed comparable methods discussed in the related work. 
Influence is calculated in two ways to characterize the difference between instance-specific and general measurements. 
In the cases labeled ``Individual'', we measure influence for each instance individually and rank features and units accordingly, whereas those labeled ``Mean'' rank features and units by influence measured over the distribution of instances in the appropriate class.

Comparing the individual and mean results tells us how well the components identified as relevant by the influence measurements generalize across the class. 
If the individual cases significantly outperform their respective mean cases, then we might conclude that the distributional influences failed to identify concepts that are relevant across the class. 
The results in Figure~\ref{fig:meanvindividual} show a very small gap in performance between the individual and mean cases for internal influence, but, unsurprisingly, this was not the case for input influence. 
This suggests that units deemed relevant to the class on-average also tend to contribute consistently across instances in that class. 
Moreover, the steeper dropoff for the internal influence measurements indicates that the identified units correspond to highly-relevant concepts with a greater degree of class-specificity.

\subsection{Validating the ``Essence'' of a Class}
\label{exp:model_comp}

The steep dropoff for internal influences in Figure~\ref{fig:meanvindividual} suggests that it is often the case that relatively few units are highly influential towards a particular class. 
Combined with the fact that these units tend to be relevant \emph{across} the class suggests the existence of a consistent, relatively small set of units that are sufficient to predict and explain the class.
We refer to this set as the ``essence'' of the class, and validate our hypothesis by isolating these units from the rest of the model to extract a binary classifier for membership in the corresponding class.

We show that these classifiers, which we call \emph{experts}, are often more proficient than the original model at distinguishing 
instances of the class from other classes in the distribution, despite comprising fewer units than the original model.
Furthermore, the performance of the original model can be achieved by experts using as few as 1\% of the available internal units.
Finally, we show that experts derived by using activation levels rather than influence measurements to identify the ``essence'' are not as effective for a fixed number of units, demonstrating that explanations based on activations are not as effective at identifying and isolating learned concepts.

\subsubsection{Class-specific experts}
Given a model $f$ with softmax output, and slice $\langle g, h\rangle$ where $g : \mathcal{Z}\to\mathcal{Y}$, let $M_h \in \mathcal{Z}$ be a 0-1 vector. 
Intuitively, $M_h$ masks the set of units at layer $h$ that we wish to retain, and thus is $1$ at all locations corresponding to such units and $0$ everywhere else. 
Then the \emph{slice compression} $f_{M_h}(X) = g(h(X) * M_h)$ corresponds to the original model after discarding all units at $h$ not selected by $M_h$. 
Given a model $f$, we obtain a binary classifier, $f^{i}$, for class $L_i$ (corresponding to output $i$) by taking the argmax over outputs and combining all classes $j \neq i$ into one class, $\neg i$; i.e., $f^i$ predicts $i$ when $f$ predicts $i$, and $\neg i$ when $f$ predicts $j \neq i$. 

\begin{figure}[t]
  \centering
  \begin{subfigure}{0.5\textwidth}
    \centering
    \includegraphics[height=.55\textwidth]{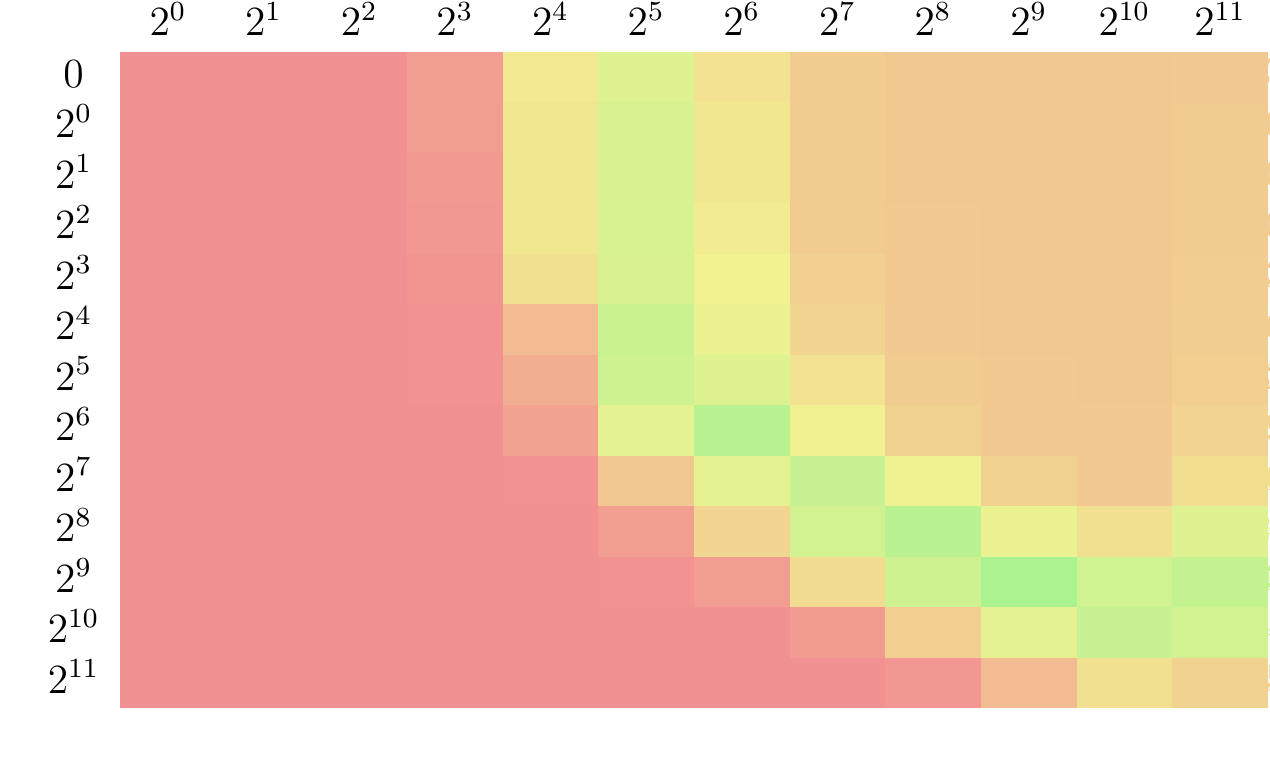}
    \caption{ }
    \label{fig:comp-space}
  \end{subfigure}

  \begin{subfigure}{0.5\textwidth}
    \centering
    \begin{tabular}{l|c|c|c}\hline
      \textbf{Class} & \textbf{Orig.} & \textbf{Infl.} & \textbf{Act.} \\\hline
      Chainsaw (491)   &  .14 & .71 & .21 \\\hline
      Bonnet (452)     &  .62 & .92 & .77 \\\hline
      Park Bench (703) &  .52 & .71 & .63 \\\hline
      Sloth Bear (297) &  .36 & .75 & .44 \\\hline
      Pelican (144)    &  .65 & .95 & .79 \\\hline
    \end{tabular}
    \caption{ }
    \label{fig:experts_table}
  \end{subfigure}
  \caption{
    \textbf{(a)} F${}_1$ score for experts derived from the first fully-connected layer of the VGG16 network on a randomly-selected ImageNet class. The rows and columns correspond to $\beta$ and $\alpha$ respectively. The layer contains 4096 neurons, so the bottom right corner corresponds to the entire network. High F${}_1$ scores are shown in green, and low scores in red. \textbf{(b)} Model compression recall for five randomly-selected ImageNet classes. Columns marked Orig. correspond to the original model, Infl. to experts computed using influence measures, and Act. to experts computed using activation levels. Precision in all cases was 1.0.
  }
\end{figure}

A class-wise \emph{expert} for $L_i$ is a slice compression $f_{M_h}$ whose corresponding binary classifier $f_{M_h}^{i}$ achieves better recall on $L_i$ than the binary classifier $f^{i}$, while achieving comparable or better precision. 
To derive an expert, we compute $M_h$ by measuring the slice influence (Equation~\ref{eq:infl-slice}) over $P_i$ using the quantity of interest $g |_{i}$. 
We then select $\alpha$ units at layer $h$ with the greatest positive influence, and $\beta$ units with the lowest negative influence (i.e., greatest magnitude among those with negative influence). 
$M_h$ is then defined to be zero at all positions except those corresponding to these $\alpha+\beta$ units. 
In our experiments, we obtain concrete values for $\alpha$ and $\beta$ by a parameter sweep, ultimately selecting parameter values that yield the best experts by recall rate. 

Figure~\ref{fig:comp-space} shows the F${}_1$ score obtained on a randomly-selected class as a function of $\alpha$ and $\beta$. 
Figure~\ref{fig:experts_table} shows the recall of experts found in this way for five randomly selected ImageNet classes.
Notably, the $\alpha$ and $\beta$ yielding the best performance correspond to less than a quarter of the units available, and the resulting expert achieves significantly better performance than the original model. 
Additionally, the performance of the original model can be matched using a tiny fraction of the available neurons (as few as 1\%), supporting the claim that the network's behavior on the class can be effectively summarized by identifying a small number of the most influential units for that class. 

\subsubsection{Inadequacy of activation levels}
Some recent prior work~\cite{oramas, deconv} uses unit activation levels to determine relevance when identifying concepts.
Here we consider an alternative approach for deriving experts by measuring the average activation at $h$ across the distribution of interest to compute $M_h$, and ranking units by average activation level.
Figure~\ref{fig:experts_table} shows the best recall of the resulting activation-based experts, and we see that activations are considerably less effective than influences for finding good experts. 
Moreover, experts derived from activations are unable to match the original model performance without using \emph{at least half} of the available units, and those with small $\alpha, \beta$ (close to 1\%) achieve zero recall in every case we evaluated. 
This appeals to the intuition that a unit may be highly active on an instance without necessarily contributing to the prediction outcome, and suggests that activation levels are not a consistent proxy for the relevance of a neuron.


\section{Explaining Instances}
\label{section:slice_exp}

In this section we demonstrate that our influence measure is also useful when explaining model behavior by instantiating general information measured across a distribution of interest to an \emph{individual} instance. 
We begin by noting that our measure generalizes previous gradient-based influence measures, so it can be parameterized to produce the same sorts of explanations shown in prior work. 
We then introduce two parameterizations that yield new sorts of explanations, showing the broader potential for our work in practical settings.
In particular, we show that \emph{(1)} internal influence can be leveraged to gain a more complete understanding of a model's decision on an instance by breaking the influential features into high-level components recognized by the model, and \emph{(2)} changing the quantity of interest yields explanations specific to how the model distinguishes between related classes on specific instances.

\subsection{Focused Explanations from Slices}
\label{sec:individual-slice}
Slice influence (Equation~\ref{eq:infl-slice}) characterizes the extent to which neurons in an intermediate layer are relevant to a particular network behavior.
We can construct explanations by using existing visualization techniques~\cite{integratedGrads} to interpret the concepts represented by internal units that are distinguished by high slice influence on an appropriate quantity of interest.
These explanations allow us to decompose the influential input features into high-level concepts recognized by the model.

\label{section:comparative}

\begin{figure}[t]
\centering
\begin{subfigure}{.14\textwidth}
\centering
\includegraphics[width=\textwidth]{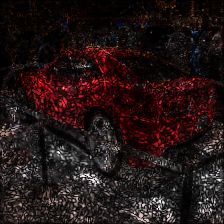}
\caption{ }
\label{fig:saliency}
\end{subfigure}
\begin{subfigure}{.423\textwidth}
\centering
\footnotesize
\includegraphics[width=\textwidth]{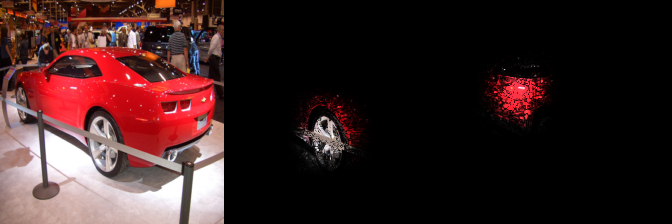}
\caption{ }
\label{fig:slice}
\end{subfigure}%
\hfill
\begin{subfigure}{.423\textwidth}
\centering
\footnotesize
\includegraphics[width=\textwidth]{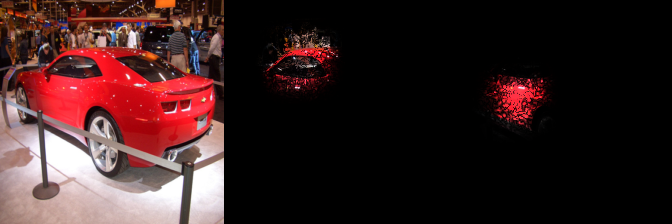}
\caption{ }
\label{fig:comp}
\end{subfigure}
\caption{
  \textbf{(a)} Input influence in the VGG16~\cite{Simonyan14c} network parameterized to match saliency maps~\cite{saliency}.
  \textbf{(b)} Interpretation of the two most influential units from the convolutional layer \texttt{conv4\_1}.
  \textbf{(c)} Comparative explanation visualizing the top two units at the \texttt{conv4\_1} layer that distinguish `sports car' from `convertible.'
}
\label{fig:visualizations}
\end{figure}

Figure~\ref{fig:slice} shows the results of interpreting the influences taken on a slice of the VGG16~\cite{Simonyan14c} network corresponding to an intermediate convolutional layer (\texttt{conv4\_1}). In this example we visualize the two most influential units for the quantity of interest characterizing correct classification of the image shown on the left of Figure~\ref{fig:slice} (sports car). More precisely, the quantity of interest used in this example is $f |_{L}$, i.e., the projection of the model's softmax output to the coordinate corresponding to the correct label $L$ of this instance. The interpretation for each of these units was then obtained by measuring the influence of the input pixels on these units along each color channel, and scaling the pixels in the original image accordingly~\cite{integratedGrads}.

Because convolutional units have a limited receptive field, the resulting interpretation shows distinct regions in the original image, in this case corresponding to the wheel and side of the car, that were most relevant to the model's predicted classification. When compared to the explanation provided by input influence, e.g., as shown in Figure~\ref{fig:saliency}, it is evident that the explanation based on the network's internal units more effectively localizes the features used by the network in its prediction.

\subsection{Comparative Explanations}
\label{sec:comparative}

Influence-directed explanations are parameterized by a quantity of interest, corresponding to the function $f$ in Equation~\ref{eq:infl-slice}. Changing the quantity of interest gives additional flexibility in the characteristic explained by the influence measurements and interpretation. One class of quantities that is particularly useful in answering counterfactual questions such as, ``Why was this instance classified as $L_1$ rather than $L_2$?'', is given by the \emph{comparative quantity}. Namely, if $f$ is a softmax classification model that predicts classes $L_1, \ldots, L_n$, then the comparative quantity of interest between classes $L_i$ and $L_j$ is $f|_{i} - f|_{j}$. When used in Equation~\ref{eq:infl-slice}, this quantity captures the tendency of the model to classify instances as $L_i$ over $L_j$.

Figure~\ref{fig:comp} shows an example of a comparative explanation. The original instance shown on the left of Figure~\ref{fig:comp} is labeled as `sports car.' We measured influence using a comparative quantity against the leaf class `convertible,' using a slice at the \texttt{conv4\_1} convolutional layer. The interpretation was computed on the top two most influential units at this layer in the same way as discussed in Section~\ref{sec:individual-slice}.

As in the examples from Figure~\ref{fig:concepts}, the receptive field of the most influential unit corresponds to the region containing the hard top of the vehicle, which is understood to be its most distinctive feature according to this comparative quantity. While both the explanations from Figure~\ref{fig:slice} and Figure~\ref{fig:comp} capture features common to cars, only the comparative explanation isolates the elements of the feature space distinctive to the type of car.

\section{Axiomatic Justification of Measures}
\label{sec:axioms}

In this section we justify the family of measures presented in Section~\ref{sec:influence} by defining a set of natural axioms for influence measures in this setting, and then proving a tight characterization.
We first address the case where the influence is measured with respect to inputs, i.e. when a slice is $f$ paired with the identity function, and then generalize to internal layers. This approach is inspired, in part, by axiomatic justification for power indices in cooperative game theory \cite{roth1988shapley,aumannShapley}--an approach that has been previously employed for explaining predictions of machine learning models~\cite{QII,integratedGrads}. An important difference, as we elaborate below, is that we carefully account for distributional faithfulness in this work.

%

\subsection{Input Influence}
\label{sec:input}


We first characterize a  measure $\infl_i(f,P)$ that measures the
influence of input $i$ for a quantity of interest $f$, and distribution of interest $P$.
The first axiom, \emph{linear agreement} states that for linear systems, the coefficient of
an input is its influence. Measuring influence in linear models is straightforward since
a unit change in an input corresponds to a change in the output given by the coefficient.

\begin{axiom}[Linear Agreement]\label{ax:la}
  For linear models of the form $f(\x) = \sum_i \alpha_ix_i$, $\infl_i(f, P) = \alpha_i$.
\end{axiom}

The second axiom, \emph{distributional marginality} is inspired by the \emph{marginality principle}~\cite{asAxioms}
in prior work on cooperative game theory. The marginality principle states
that an input's importance only depends on its own contribution to the output.
Formally, if the partial derivatives with respect to an input
of two functions  are identical at all input instances, then that input is equally important for
both functions.

Our axiom of distributional marginality (DM) is a weaker form of this requirement that only requires
equality of importance when partial derivatives are same for points in the support of the distribution.
This axiom ensures that the influence measure only depends on the
behavior of the model on points within the manifold containing the input
distribution. Such a property is important for deep learning systems since the
behavior of the model outside of this manifold is unpredictable.

\begin{axiom}[Distributional marginality (DM)]\label{ax:dm}
  If $$P\left(\pderivat{f_1}{x_i}{X} = \pderivat{f_2}{x_i}{X}\right) = 1,$$ where $X$ is the random variable over instances from $\X$, then $\infl_i(f_1, P) = \infl_i(f_2, P)$.
\end{axiom}

The third axiom, \emph{distribution linearity} states that the influence measure is linear
in the distribution of interest. This ensures that influence measures are properly weighted over the input space, i.e.,
influence on infrequent regions of the input space receive lesser weight in the influence measure as compared to more frequent regions.

\begin{axiom}[Distribution linearity (DL)]\label{ax:dl}
  For a family of distributions indexed by some $a \in \cal A$,
  $P(x) = \int_{\cal A} g(a)P_a(x)da$, then $\infl_i(f, P) = \int_{\cal A} g(a)\infl_i(f, P_a)da$.
\end{axiom}


\begin{theorem}
  The only measure that satisfies linear agreement, distributional marginality and distribution linearity is given by
  $$
    \infl_i(f, P) =  \int_\X \pderivat{f}{x_i}{\x} P(\x) d\x.
  $$
\end{theorem}

\begin{proof}
  Choose any function $f$ and $P_{\vec a}(\x) = \delta(\x - \vec a)$, where
  $\delta$ is the Dirac delta function on $\X$. Now, choose $f'(\x) =
  \pderiv{f}{\x_i}|_{\vec a} x_i$. By \axla, it must be the case that,
  $\infl(f', P_{\vec a}(\x)) = \pderiv{f}{x_i}|_{\vec a}$. By distributional marginality, we
  therefore have that $\infl_i(f, P_{\vec a}) = \infl_i(f', P_{\vec a}) =
  \pderiv{f}{x_i}|_a$. Any distribution $P$ can be written as $P(\x) = \int_\X
  P({\vec a})P_{\vec a}(\x)d{\vec a}$. Therefore, by the distribution linearity axiom, we have that
  $\infl(f, P) = \int_X P({\vec a})\infl(f, P_a) da = \int_\X P({\vec a})\pderiv{f}{x_i}|_{\vec a} d{\vec a}$.
\end{proof}


\subsection{Internal influence}
\label{sec:internal}

In this section, we generalize the above measure of input influence to a measure that can be used to measure the
influence of an internal neuron.
We again take an axiomatic approach, with two natural invariance properties on
the structure of the network.

The first axiom states that the influence measure is
agnostic to how a network is sliced, as long as the
neuron with respect to which influence is measured
is unchanged. Below, the notation $\x_{-i}$ refers
to the vector $\x$ with element $i$ removed and $\x_{-i}y_i$ is the
vector $\x$ with the $i^{\text{th}}$ element replaced with $y_i$.

Two slices, $s_1 = \langle g_1, h_1\rangle$ and $s_2 = \langle g_2, h_2\rangle$, are $j$-equivalent if for all $\x\in \X$, and $z \in \Z$, $h_1(\x)_j = h_2(\x)_j$, and $g_1(h_1(\x)_{-j} z_j) = g_2(h_2(\x)_{-j} z_j)$. Informally, two slices are $j$-equivalent as long as they have the same function for
representing $z_j$, and the causal dependence of the outcome on $z$ is identical.

\begin{axiom}[Slice Invariance]
For all $j$-equivalent slices $s_1$ and $s_2$, $\infl_j^{s_1}(f, P) = \infl_j^{s_2}(f, P)$.
\end{axiom}

The second axiom equates the input influence of an input with the internal influence of a perfect predictor of that input. Essentially, this encodes a consistency requirement between inputs and internal neurons that if an internal neuron has exactly the same behavior as an input, then the internal neuron should have the same influence as the input.

\begin{axiom}[Preprocessing] Consider $h_i$ such that $P(X_i =  h_i(X_{-i})) = 1$. Let $s = \langle f_1, h\rangle$, be such that $h(\x_{-i}) = \x_{-i}h_i(\x_{-i})$, which is a slice of $f_2(\x_{-i}) = f_1(\x_{-i}h_i(\x_{-i}))$, then $\infl_i(f_1, P) = \infl_i^s(f_2, P)$.
\end{axiom}

 We now show that the only measure that satisfies these two properties is the one presented above
 in Equation~\ref{eq:infl-slice}.
 First, we prove the following lemma that shows that expected gradient computed at a slice can be
 computed with either the probability distribution at the input or the slice.

\begin{lemma}
\label{lem:switch}
Let $s=\langle g,h\rangle$ be a slice for $f$.
Given distribution $P_\X(\x)$ on $\X$,
let $P_\Z(\z)$ be the probability distribution induced by applying $h$ on $\x$, given by:

\[P_\Z(\z) = \int_\X P_\X(\x)\delta(h(\x) - \z)d\x.\]
Then $\infl_j(g,P_\Z) = \infl_j^s(f,P_\X)$.

\end{lemma}

\begin{proof}
\begin{align}\infl_j(g,P_\Z) =& \int_\Z \pderivat{g}{z_j}{\z}P_\Z(\z)d\z \\
                           =& \int_\Z \pderivat{g}{z_j}{\z}\int_\X P_\X(\x)\delta(h(\x) - \z)d\x d\z \\
                           =& \int_\X P_\X(\x)\int_\Z \pderivat{g}{z_j}{\z} \delta(h(\x) - \z) d\z d\x  \\
                           =& \int_\X \pderivat{g}{z_j}{h(\x)} P_\X(\x) d\x\\
                           =& \infl_j^s(f,P_\X)
\end{align}

\end{proof}

\begin{theorem}
The only measure that satisfies slice invariance and preprocessing is Equation~\ref{eq:infl-slice}.
\end{theorem}

\begin{proof}
Assume that two slices $s_1 = \langle g_1, h_1\rangle$ and $s_2 = \langle g_2, h_2 \rangle$ are $j$-equivalent. Therefore, $g_1(h_1(\x)_{-j}z_j) = g_2(h_2(\x)_{-j}z_j)$. Taking partial derivatives
with respect to $z_j$, we have that:
\[\pderivat{g_1}{z_j}{h_1(\x)_{-j}z_j} = \pderivat{g_2}{z_j}{h_2(\x)_{-j}z_j}\]
Now, since $h_1(\x)_j = h_2(\x)_j$, we have that
\[\pderivat{g_1}{z_j}{h_1(\x)} = \pderivat{g_2}{z_j}{h_2(\x)}\]
Plugging the derivatives into \ref{eq:infl-slice}, we get that $\infl_j^{s_1}(f, P) = \infl_j^{s_2}(f, P)$, and that the measure satisfies slice invariance.

Consider $h_i$ such that $P(X_i =  h_i(X_{-i})) = 1$. Let $s = \langle f_1, h\rangle$, be such that $h(\x_{-i}) = \x_{-i}h_i(\x_{-i})$, which is a slice of $f_2(\x_{-i}) = f_1(\x_{-i}h_i(\x_{-i}))$.
\begin{align}
  \infl_i^s(f_2,P) =& \int_\X \pderivat{f_1}{x_i}{\x_{-i}h(\x_{-i})} P(\x) d\x \\
                     =& \int_\X \pderivat{f_1}{x_i}{\x} P(\x) d\x \\
                     =& \infl_i(f_1, P).
\end{align}
Therefore, the measure satisfies preprocessing.

For the opposite direction, consider any slice $s = \langle g,h\rangle$ of $f$. We
wish to show that if $\infl_j^s(f,P_\X)$ satisfies slice invariance and preprocessing, then $\infl_j^s(f,P_\X) = \int_{\X}\pderivat{g}{z_j}{h(\x)}P(\x)d\x$.
Consider the slice $s' = \langle g',h' \rangle$ such that $h'(\x) = (\x, h_j(\x))$, and $g'(\x, z_j) = g(h_{-j}(\x)z_j)$. Essentially $s'$ is a slice of $f$ that only processes $h_j(\x)$.
By Lemma~\ref{lem:switch},
$\infl_j(g', P_{\X}) = \int_{\X}\pderivat{g}{z_j}{h(\x)}P(\x)d\x$.
By preprocessing $\infl^{s'}_j(f, P_{\X}) = \infl_j(g', P_{\X})$. As $s$ and $s'$ are $j$-equivalent,
$\infl^{s}_j(f, P_{\X}) = \infl^{s'}_j(f, P_{\X})$.
\end{proof}

%


\section{Related Work}
\label{sec:related}
\begin{table*}[t]
\centering
\normalsize
\begin{tabular}{>{\raggedright}m{25ex}c @{\hspace{0.15cm}} c @{} c @{} c @{\hspace{0.15cm}} c}
& \multicolumn{3}{c}{\it Explanation framework properties} & \multicolumn{2}{c}{\it Influence properties} \\
& {\bf Quantity} & {\bf Distribution} & {\bf Internal} & \textbf{Marginality} & \textbf{Sensitivity} \\
\hline
Influence-Directed & \checkmark & \checkmark & \checkmark & \checkmark & \checkmark${}^{*}$ \\

Integrated Gradients~\cite{integratedGrads} & & \checkmark${}^{-}$ & & \checkmark & \checkmark \\

Simple Taylor~\cite{bach-plos15} & & \checkmark${}^{-}$ & & \checkmark & \\

Sensitivity Analysis~\cite{saliency} & & & & \checkmark & \\

Deconvolution~\cite{deconv} & & & $\checkmark^{\dagger}$ & & \\

Guided Backpropagation~\cite{backprop} & & & $\checkmark^{\dagger}$ & \checkmark & \\

Relevance Propagation~\cite{bach-plos15} & & \checkmark${}^{-}$ & $\checkmark^{\dagger}$ & \checkmark${}^{*}$ & \checkmark${}^{*}$ \\[.5em]

\end{tabular}

\caption{Comparison of the influence-directed explanations proposed here to
  prior related work.
  $\checkmark^-$ denotes that the framework has limited flexibility for the feature, $\checkmark^*$ denotes that the framework may have the feature under certain parameterizations, and $\checkmark^\dagger$ denotes that the framework measures internal influence only as an intermediary step to computing feature influence.
}\label{tab:comparison}
\end{table*}

We begin by pointing out some high-level differences between our work and other approaches as shown in Table~\ref{tab:comparison}.
We then discuss important specific differences in more detail below.

The leftmost three columns of Table~\ref{tab:comparison} describe properties on which explanation techniques differ, and on which our approach is parameterized.
First, our approach is parametric in a \textbf{Quantity} of interest that allows us to provide explanations for different behaviors of a system, as opposed to simply explaining absolute instance predictions.
Second, we can specify a \textbf{Distribution} of interest, allowing explanations of network behavior across different groups of instances (e.g., an instance or a particular class).
Cells marked $\checkmark^{-}$ in these columns denote limited flexibility along this dimension through the choice of a baseline, as in integated gradients~\cite{integratedGrads}.
Finally, our approach can select which \textbf{Internal} neurons to measure, which, as we demonstrate in Section~\ref{sec:experiments}, is key to identifying learned concepts.
By contrast, integrated gradients~\cite{integratedGrads}, sensitivity analysis~\cite{saliency}, and simple Taylor decomposition~\cite{bach-plos15} assign importance solely to the input features. 
Deconvolution~\cite{deconv}, guided backpropagation~\cite{backprop}, and layer-wise relevance propagation\cite{bach-plos15}, use internal influence in the course of computing input influence, but do not apply internal influence measurements to identifying learned concepts.

The rightmost two columns in Table~\ref{tab:comparison} describe properties of the influence measure used to build explanations.
\textbf{Marginality} requires that the influence of each feature depends only on its own marginal contribution, which is implied by distributional marginality.
Measures not satisfying marginality may attribute behavior to the wrong features, giving misleading results.
\textbf{Sensitivity} requires that if the instance and a \textit{baseline} instance differ in one feature and yield different predictions, then that feature is assigned non-zero influence.
Because sensitivity refers to a baseline, our explanations must specify the baseline via the distribution of interest to achieve this property.
Measures failing to satisfy sensitivity may fail to identify features that are causally relevant to the explanation, leading to ``blind spots'' and misleading results.

\subsubsection{Identifying influential regions}
One approach to interpreting predictions for convolutional networks is to map activations of neurons back to regions in the input image that are the most relevant to the outcomes of the neurons.
Possible approaches for localizing relevance include: \emph{(1)} visualizing gradients~\cite{saliency, integratedGrads,bach-plos15}, \emph{(2)} propagating activations back using gradients~\cite{deconv,backprop,bach-plos15}, and \emph{(3)}
fitting a simpler interpretable model around a test point to predict relevant input regions~\cite{lime}.
Because these approaches relate instance-specific features to instance-specific predictions, their results do not generalize beyond a single input point, as demonstrated in Section~\ref{sec:exp:effectiveness}.

Most prior approaches have not leveraged internal units.
Two exceptions are class activation mapping (CAM, Grad-CAM)~\cite{gradCAM, CAM}, in which objects in an image are localized by measuring the influence of feature maps, and a recent technique proposed by Oramas et al.~\cite{oramas} in which internal neurons are interpreted to provide an explanation.
CAM and Grad-CAM differ from our work in that internal influences are aggregated to represent the localization of a concept identifying an entire class in an input instance, whereas our approach is more granular and can isolate components that represent simpler concepts than an entire class.
Oramas et al.~\cite{oramas} use unit activation levels to determine relevance, which, as we demonstrate in Section~\ref{exp:model_comp}, is less effective at identifying important concepts than our influence measure.
Concurrent work~\cite{conductance} has also suggested a slightly different approach; namely, measuring the input attribution that ``flows through'' a particular internal neuron.

\subsubsection{Visualization by maximizing activation}
An orthogonal approach is to visualize learned features by identifying input instances that maximally activate a neuron, achieved by either optimizing the activation in the input space~\cite{saliency,MahendranV14,NguyenDYBC16}, or by searching for instances in a dataset~\cite{girshick14rich}.
These techniques can complement our work by providing a means to visualize the concept learned by a set of neurons that the influence measure identifies as important for a particular quantity and distribution of interest.

\subsubsection{Attribution vs. Influence}
Of the measures summarized in Table~\ref{tab:comparison}, some, e.g., Integrated Gradients~\cite{integratedGrads} and Relevance Propagation~\cite{bach-plos15}, measure \emph{attribution}, while others, e.g., Sensitivity Analysis~\cite{saliency} and Influence-Directed explanations, measure \emph{influence}.
Here, attribution can be understood as the amount of the quantity of interest that can be attributed to a particular neuron. In contrast, influence addresses the sensitivity of the quantity of interest to a particular input or input distribution.

Dhamdhere et al.~\cite{conductance} claim that measures calculating influence, such as influence-directed explanations, lead to non-intuitive results in some cases \footnote{Although this paper was released a few months after our paper on arXiv, we regard it as independent, concurrent work based on conversations with the authors.}. 
This argument is predicated on the a priori insistence on an axiom called \emph{completeness}, which states that the sum of the influences must equal the change in output relative to the baseline. We find this position difficult to get behind. First, it is unclear to us why this axiom should be demanded for influence measures given the nuanced difference between attribution and influence described above. Second, even for attribution measures (power indices) from co-operative game theory, the completeness axiom does not always hold. Straffin provides an analysis of two different power indices---one of which satisfies completeness and the other doesn't---proving that the degree of statistical independence between the inputs determines which index is appropriate for use (see \cite{roth1988shapley}, Chapter $5$).

Another challenge with applying Integrated Gradients~\cite{integratedGrads} is that it may not respect distributional faithfulness since the axioms used to arrive at that importance measure does not enforce such a constraint. Kindermans et al.~\cite{unreliability} argue that measures calculating attribution may give undesirable explanations when the baseline is not appropriately selected to control for trends in the dataset. 

We suspect that both attribution and influence may have complementary applications.
Future work may help determine for which applications attribution or influence is more effective, and for which attribution and influence complement one another.




%

\section{Future Work}
\label{sec:future-work}
We expect the distributional influence measure introduced in this paper to be applicable to a broad set of deep neural networks. One direction for future work is to couple this measure with appropriate interpretation methods to produce influence-directed explanations for other types of deep networks, such as recursive networks for text processing tasks. Another direction is to develop debugging tools for models using influence-directed explanations as a building block.

\paragraph*{Acknowledgment.} This work was developed with the support of NSF grant CNS-1704845 as well as by DARPA and Air Force Research Laboratory under agreement number FA9550-17-1-0600. The U.S. Government is authorized
to reproduce and distribute reprints for Governmental purposes not
withstanding any copyright notation thereon. The views, opinions, and/or
findings expressed are those of the author(s) and should not be interpreted as
representing the official views or policies of DARPA, the Air Force Research
Laboratory, the National Science Foundation, or the U.S. Government.

\bibliographystyle{unsrt}
\bibliography{cnn_explanations}

\end{document}